\documentclass[dvipsnames]{article}
\usepackage{arxiv}

\usepackage[svgnames]{xcolor}
\usepackage{tikz}
\usetikzlibrary{decorations.markings}
\usetikzlibrary{shapes.geometric}
\usetikzlibrary{arrows, positioning, quotes}

\pgfdeclarelayer{edgelayer}
\pgfdeclarelayer{nodelayer}
\pgfsetlayers{edgelayer,nodelayer,main}

\usepackage[utf8]{inputenc}
\usepackage{tikz}
\usepackage{amsmath}
\usepackage{amsthm}
\usepackage[toc,page]{appendix}
\usepackage{amssymb}
\usepackage{url}
\usepackage{enumitem}
\usepackage{dsfont}
\usepackage{breakcites}
\usepackage{xcolor}
\usepackage[colorlinks=true,breaklinks]{hyperref}

\usetikzlibrary{arrows, positioning, quotes}

\newcommand{\R}{\mathbb{R}}

\newcommand{\X}{{\cal X}}

\newcommand{\Y}{\cal Y}
\newcommand{\XY}{\X\times\Y}

\newcommand{\ie}{i.e.}

\newtheorem{lemma}{Lemma}
\newtheorem{op}{Open Problem}
\newtheorem*{lem}{Lemma}

\newtheorem{theorem}{Theorem}

\newtheorem{corollary}{Corollary}
\newtheorem{remark}{Remark}
\newenvironment{proof_idea}{%
	\proof}{\endproof}

\newcommand{\D}{{\cal D}}

\tikzstyle{arrow} = [draw, -latex']

\usetikzlibrary{arrows.meta, calc}

\allowdisplaybreaks

\tikzstyle{none}=[inner sep=0pt]
\definecolor{hexcolor0xa0fdaa}{rgb}{0.627,0.992,0.667}
\definecolor{hexcolor0xa7fab7}{rgb}{0.655,0.980,0.718}
\definecolor{hexcolor0xefb4b4}{rgb}{0.937,0.706,0.706}
\definecolor{hexcolor0xedadad}{rgb}{0.929,0.678,0.678}
\definecolor{hexcolor0xbcaff5}{rgb}{0.737,0.686,0.961}
\definecolor{hexcolor0xb3abf1}{rgb}{0.702,0.671,0.945}
\definecolor{hexcolor0x030000}{rgb}{0.012,0.000,0.000}
\definecolor{hexcolor0x0deb23}{rgb}{0.051,0.922,0.137}
\definecolor{hexcolor0x14f110}{rgb}{0.078,0.945,0.063}
\definecolor{hexcolor0xfd978b}{rgb}{0.992,0.592,0.545}
\definecolor{hexcolor0xf5ad9c}{rgb}{0.961,0.678,0.612}
\definecolor{hexcolor0xf6bd8b}{rgb}{0.965,0.741,0.545}
\definecolor{hexcolor0xf6c375}{rgb}{0.965,0.765,0.459}
\definecolor{hexcolor0x7360e6}{rgb}{0.451,0.376,0.902}
\definecolor{hexcolor0xf27373}{rgb}{0.949,0.451,0.451}
\definecolor{hexcolor0xf1a269}{rgb}{0.945,0.635,0.412}

\tikzstyle{rn}=[circle,fill=Red,draw=Black,line width=0.8 pt]
\tikzstyle{gn}=[circle,fill=Lime,draw=Black,line width=0.8 pt]
\tikzstyle{yn}=[circle,fill=Yellow,draw=Black,line width=0.8 pt]
\tikzstyle{lu}=[circle,fill=hexcolor0xa0fdaa,draw=hexcolor0xa7fab7]
\tikzstyle{ne}=[circle,fill=hexcolor0xefb4b4,draw=hexcolor0xedadad]
\tikzstyle{li}=[circle,fill=hexcolor0xbcaff5,draw=hexcolor0xb3abf1]
\tikzstyle{ne}=[circle,fill=hexcolor0xefb4b4,draw=hexcolor0xedadad]
\tikzstyle{basic}=[circle,fill=hexcolor0x030000,draw=Black]
\tikzstyle{basic2}=[circle,fill=hexcolor0xf5ad9c,draw=hexcolor0xf5ad9c]
\tikzstyle{gg}=[circle,fill=hexcolor0x0deb23,draw=hexcolor0x14f110]
\tikzstyle{newstyle}=[rectangle,fill=hexcolor0xfd978b,draw=hexcolor0xf5ad9c]
\tikzstyle{nine}=[rectangle,fill=hexcolor0xf6bd8b,draw=hexcolor0xf6c375]
\tikzstyle{annot} = [text width=4.8em, text centered]

\tikzstyle{pl}=[->,draw=hexcolor0x7360e6,line width=1.100]
\tikzstyle{mo}=[->,draw=hexcolor0xf27373,line width=1.100]

\tikzstyle{simple}=[-,draw=Black,line width=2.000]
\tikzstyle{arrow}=[-,draw=Black,postaction={decorate},decoration={markings,mark=at position .5 with {\arrow{>}}},line width=2.000]
\tikzstyle{tick}=[-,draw=Black,postaction={decorate},decoration={markings,mark=at position .5 with {\draw (0,-0.1) -- (0,0.1);}},line width=2.000]
\tikzstyle{newt}=[->,draw=Black]
\tikzstyle{p}=[->,draw=hexcolor0x7360e6,line width=1.600]
\tikzstyle{nina}=[->,draw=hexcolor0xf1a269,line width=0.8]
\tikzstyle{pot}=[-,draw=Black]
\tikzstyle{pol}=[->,draw=Black,line width=1.400]

\title{Deep Neural Networks Are Congestion Games:\\ From Loss Landscape to Wardrop Equilibrium and Beyond}

\author{Nina Vesseron\\ENS Lyon,\\F-69000, Lyon, France\\ \url{name.surname@ens-lyon.fr} \And Ievgen Redko, Charlotte Laclau \\
  Univ Lyon, UJM-Saint-Etienne, CNRS, Institut d'Optique Graduate School\\
  Laboratoire Hubert Curien UMR 5516, F-42023, Saint-Etienne, France\\
  \url{name.surname@univ-st-etienne.fr}}
  
\begin{document}
\maketitle
%

%

\begin{abstract}
    The theoretical analysis of deep neural networks (DNN) is arguably among the most challenging research directions in machine learning (ML) right now, as it requires from scientists to lay novel statistical learning foundations to explain their behaviour in practice. While some success has been achieved recently in this endeavour, the question on whether DNNs can be analyzed using the tools from other scientific fields outside the ML community has not received the attention it may well have deserved. In this paper, we explore the interplay between DNNs and game theory (GT), and show how one can benefit from the classic readily available results from the latter when analyzing the former. In particular, we consider the widely studied class of congestion games, and illustrate their intrinsic relatedness to both linear and non-linear DNNs and to the properties of their loss surface. Beyond retrieving the state-of-the-art results from the literature, we argue that our work provides a very promising novel tool for analyzing the DNNs and support this claim by proposing concrete open problems that can advance significantly our understanding of DNNs when solved. 
\end{abstract}

\allowdisplaybreaks
\section{Introduction}
Since the very seeding of the machine learning (ML) field, the ML researchers has constantly drawn inspiration from other areas of science both to develop novel approaches and to better understand the existing ones. One such notable example is a longstanding fruitful relationship of ML with game theory (GT) that manifested itself by the novel insights regarding the analysis of such different learning settings as reinforcement learning \cite{reinforc_1,reinforc_2,reinforc_3}, boosting \cite{freund96} and adversarial classification \cite{advers_1,Bruckner2011,advers_2} to name a few.  While the interplay between ML and GT in the above-mentioned cases is natural, \textit{ie}, reinforcement learning is a game played between the agent and the environment, boosting is a repeated game with rewards and adversarial learning can be seen as a traditional minimax game, very few works studied the connection between the deep neural networks (DNNs) and GT despite the omnipresence of the former in the ML field. Indeed, in recent years, deep learning has imposed itself as the state of the art ML method in many real-world tasks, such as computer vision or natural language processing to name a few \cite{Goodfellow-et-al-2016}. While achieving impressive performance in practice, training DNNs requires optimizing a non-convex non-concave objective function even in the case of linear activation functions and can potentially lead to local minima that are arbitrary far from global minimum. This, however, is not the typical behaviour observed in practice, as several works \cite{dauphin_14,GoodfellowV14} showed empirically that even in the case of training the state-of-the-art convolutional or fully-connected feedforward  neural networks one does not converge to suboptimal local minima. Such a mysterious behaviour made studying the loss surface of DNNs and characterizing their local minima one of the topics of high scientific importance for the ML community.  

In this paper, we propose a novel approach for analyzing DNNs' behaviour by modelling them as congestion games, a popular class of games first studied by \cite{Rosenthal1973} in the context of traffic routing. To this end, we first prove that linear DNNs can be cast as special instances of non-atomic congestion games defined entirely in terms of the DNNs main characteristics. This result allows us to successfully draw the parallel between local minima of the loss function of a linear DNN and the Wardrop equilibria of the corresponding non-atomic congestion game under some mild assumption on the considered loss function. As a consequence, we prove the well-known result provided by \cite{kawaguchi2016deep} for linear networks regarding the equivalence between the local minima of the loss function optimized by a linear DNN and its global optimum. Second, we study the case of non-linear DNNs with rectified linear activation function (ReLU) by considering the model proposed in the seminal work of \cite{choromanska2014loss}. We model such networks as congestion games where some resources available to the agents in the game can fail. In this latter setting, we show that the seminal model of \cite{choromanska2014loss} is essentially equivalent to the linear DNN model studied before and thus enjoys the same guarantees. To the best of our knowledge, the proposed approach for the analysis of DNNs has never been studied in the literature before and we expect it to have a very strong scientific impact due to the established formal connection between one of the most studied ML models and one of the most rich areas of GT in terms of the number of available results.

The rest of this paper is organized as follows. In Section 2, we review other existing works on the analysis of DNNs and their loss surface and convergence properties. In Section 3, we provide the required preliminary knowledge related to both DNNs and congestion games. Section 4 contains our main contributions that analyze both linear and non-linear DNNs. Finally, in Section 5 we emphasize the importance of the established connection between DNNs and congestion games and pose several open problems. 


\section{Related work}
Below, we briefly review the main related works with a particular emphasis on contributions analyzing the loss surface of DNNs\footnote{For more results on the theoretical analysis of gradient descent convergence for over-parametrized models including Recurrent and Convolutional NNs, we refer the interested reader to \cite[Section 2]{DuJLJSP17} and \cite[Section 1.2]{Allen-ZhuLS19}.} and those linking DNNs to GT. 
\paragraph{Analysis of DNNs' loss surface} 
While strong empirical performance of DNNs make of them a number one choice for many ML practitioners, it has been shown that training a neural network is NP-hard \cite{Blum1992} as it requires finding a global minimum of a non-convex function of high dimensionality. To circumvent this difficulty, the methods for convex optimization are widely used to train DNNs, but the reasons why these methods work well in practice remain unknown as, in principle, nothing restricts them from converging to poor local minima arbitrary far from the global minimum. To shed light on this, several works adapted a geometric approach to provide a justification for optimizing DNNs using convex optimization methods. This latter consists in studying the general class of non-convex optimization problems with desired geometric properties, \ie, equivalence of local minima to global optimum and negative curvature for every saddle point, and showing that DNNs belong to this class. In the case of the linear networks, such notable result was provided in the works of \cite{baldi89} and \cite{kawaguchi2016deep} who proved that local minima are global minimum when a squared loss function is considered. This statement was proved for non-linear networks as well, first by \cite{choromanska2014loss} who showed that the number of bad local minima decreases quickly with the distance to the global optimum and then by several other more recent follow-up works considering different NN's configurations \cite{HardtM17,FreemanB17,SoudryH18,SafranS18}. An important consequence for DNNs having these properties is that (perturbed) gradient descent provably converges to a global optimum in this case \cite{GeHJY15,Jin0NKJ17,DuJLJSP17}. Contrary to the works using the above-mentioned geometric approach, we obtain the same results for a family of loss functions resembling the squared loss relying solely on the properties of non-atomic congestion games. 

\paragraph{Game theory and ML} To the best of our knowledge, only two other studies have reduced learning of DNNs to game playing. In \cite{balduzzi2016deep}, the author studied DNNs with non-differentiable activation functions in order to explain why methods designed for convex optimization are guaranteed to converge on modern convnets with non-convex loss functions\footnote{While being highly insightful, this work was shown to have several flaws \cite[Supp. material, Section J]{SchuurmansZ16} that remain unaddressed up to now.}.  On the other hand, the authors of \cite{SchuurmansZ16} showed how supervised learning of a DNN with differentiable convex gates can be seen as a simultaneous move two-person zero-sum game in order to further establish the equivalence between the Karush-Kuhn-Tucker (KKT) points of a DNN and Nash equilibria of the corresponding game. With these results in hand, the authors illustrated empirically that a well-known regret matching algorithm often used to find coarse-correlated Nash equilibria can be used successfully to train DNNs. It is worth noticing that in both these papers, the games considered by the authors were designed for their specific purpose and have not been studied independently in the game theory field. On the contrary, in this paper we aim to study DNNs as instances of arguably one of the most studied classes of games in order to make the rich body of existing theoretical results proved for them readily available for the ML researchers. Also, unlike the two other papers, we study non-atomic games which are infinite-person games with each player having an infinitesimal impact on the game's analysis and Wardrop equilibria specific to such games contrary to games with a finite number of players and Nash equilibria considered before. 
\section{Background knowledge}
In this section, we briefly review the main definitions related to DNNs and congestion games. 
\paragraph{Deep Neural Networks} Let us consider a DNN defined as $N = (V, E, I, O, F)$, where 1) $V$ is a set of vertices, \ie, the total number of units in the neural network; 2) $E \subseteq V\times V$ is a set of edges; 3) $I = \{i_1, \dots, i_d\} \subset V$ is a set of input vertices equal to the number of input features; 4) $O = \{o_1, \dots, o_C\} \subset V$ is a set of output vertices of size equal to number of outputs and 5) $F = \{f_v: v \in V\}$ is a set of activation functions, where $f_v : \R \rightarrow \R$.

In the graph defined by $G = (V, E)$ and having a layered structure with $L$ layers, a path $p = (v_1, ..., v_L)$ with $v_1 \in I$ and $v_L \in O$ consists of a sequence of vertices such that $(v_j , v_{j+1}) \in E$ for all $j$.  We assume that $G$ is directed and contains no cycles, the input vertices have no incoming edges and the output vertices have no outgoing edges. 
We let $n_l$ denote the number of neurons at each layer $l \in [1,\dots,L]$ where $n_1 = d$ and $n_L = C$. We further associate a (trainable) weight $w_{ij}^{(l)}$ to an edge between vertex $v_i^{(l)}$ of layer $l$ and $v_j^{(l-1)}$ of layer $l-1$ and denote by $w^{(l)}$ the matrix of all weights between the two layers. $W=\{w^{(\ell)}, \forall \ell\}$ is the set of all parameters associated to the network. For each vertex $v_i^{(l)}$, we also associate a value (activation function) $g_i^{(l)} = f_{v_i^{(l)}} (z_i^{(l)})$ with $z_i^{(l)} = \sum_k^{n_{l-1}} w_{jk}^{(l)}g_k^{(l-1)}$.

Given a training set $\mathcal{L} = \{x^j,y^j\}_{j=1}^M$ drawn from distribution $\D$ on $\XY$ with $\X \subseteq \R^d$ and $\vert \Y \vert = C$, the task of the neural network is to produce a predictor $h:\X \rightarrow Y$ that assigns a label $y^j \in Y$ to each $x^j \in \X$. 
This is done by solving the following optimization problem:
    \begin{align}
        \min_{W} \textnormal{loss}(W) = \min_{W} \frac{1}{M}\sum_{j=1}^M\ell(o^L(x^j),y(x^j)),
        \label{eq:obj_function}
    \end{align}
where $\ell:\R\times\R\rightarrow \R_+$ is a convex loss function. Stochastic gradient descent (SGD) is commonly used to solve Problem \eqref{eq:obj_function} where the weights are updated either for each example $x$ or for a mini-batch.

\paragraph{Congestion Games} 
We consider a non-atomic version of the congestion games \cite{schmeidler1973} that were first defined in \cite{Rosenthal1973}. All along the paper, we use the definition of non-atomic congestion games from \cite{nacg_roughgarden} and use some of the results established in this paper. 

A non-atomic congestion game is defined by a 5-tuple $\textnormal{NCG} = (E(G),c,S,n,a)$ where the number of players is infinite and individual player's significance is considered to be infinitesimal. 
In such a game, $E$ is the set of resources of the game that we take to be a set of edges of a graph $G$. We associate to each element $e \in E$ a cost function $c_e(\xi)$ which is the cost of the resource $e$ as a function of the congestion $\xi$ on the edge $e$. It is assumed that each cost function is positive and continuous on the interval $[0,+\infty]$. We further consider $k$ types of players with each player of type $i \in [[k]]$ disposing a set of strategies $S_i$ given by a set of subsets of $E$ and represented by the interval $[0,n_i]$ endowed with Lebesgue measure. We associate a positive rate of consumption $a_{S,e}$ to a type of player $i$, a strategy $S \in S_i$ and an edge $e \in S$. One can see this rate as the amount of congestion that players of type $i$ contribute to element $e$ while selecting strategy $S$. In what follows, we impose $a_{S,e} = 0$ if $e \notin S$.
An action distribution is a vector $z$ of positive reals indexed by $\bigcup_{i=1}^d S_i$ such that $\sum_{S \in S_i} z_S = n_i, \forall i$. One can see $z_S$ as the measure of the set of players that selects strategy $S$. We call $\tilde{z}_e$ the total amount of congestion on element $e$ produced by the action distribution $z$:
$$\tilde{z}_e = \sum_{i=1}^d \sum_{S \in S_i} a_{S,e}z_S.$$

The cost $c_S(z)$ incurred by a player of type $i$ selecting strategy $S \in S_i$ is defined with respect to (w.r.t) the action distribution $z$ as follows:
$$c_S(z) = \sum_{e \in S} a_{S,e} c_e(\tilde{z}_e).$$
The social cost $\textnormal{SC}(z)$ w.r.t. an action distribution $z$ and the social optimum $\textnormal{SO}$ of a game are given respectively by:
$$\textnormal{SC}(z) = \sum_{i=1}^k \sum_{S \in S_i} c_S(z) z_S, \quad \text{SO} = \min_z \textnormal{SC}(z).$$
In what follows, when we speak about the value of an action distribution, we mean the value of the social cost associated to this distribution. 

An action distribution $z$ is a \textbf{Wardrop equilibrium} (\text{WE}) if for each player type $i\in [[k]]$ and strategies $S_1$, $S_2$ $\in S_i$ such that $z_{S_1}>0$, we have $c_{S_1}(z) \leq c_{S_2}(z)$.

The main results about non-atomic congestion games needed further are the followings:
 \begin{enumerate}[leftmargin = 1cm]
    \item[\textbf{P1.}] Social cost $\textnormal{SC}(z)$ can be rewritten as: 
    $$\textnormal{SC}(z) = \sum_{e \in E} c_e(\tilde{z}_e) \tilde{z}_e\text{ with } \tilde{z}_e = \sum_i \sum_{S \in S_i} a_{S,e}z_S.$$
    \item[\textbf{P2.}] Each $\textnormal{NCG}$ admits a Wardrop equilibrium.
    \item[\textbf{P3.}] All Wardrop equilibria have the same value.
    \item[\textbf{P4.}] For a given game $\textnormal{NCG}$, we define the price of anarchy (PoA) of a game as:
        $$\text{PoA}(\textnormal{NCG}) = \frac{\text{\text{WE}}(\textnormal{NCG})}{\text{SO}(\textnormal{NCG})}.$$
\end{enumerate}

\section{Main contributions}
We start by introducing the assumptions needed to model DNNs as non-atomic congestion games. We argue that these assumptions are not restrictive in practice and has been used in the literature before. Then, we proceed by formally proving the equivalence between DNNs and  non-atomic congestion games and by relating the local minima of the former to the Wardrop equilibria of the latter.
\subsection{Problem setup} 
Hereafter, we consider the following assumptions: 
\begin{enumerate}[leftmargin = 1cm]
    \item[\textbf{A1}.] $\forall i,j,l$, $w_{ij}^{(l)}\geq 0$ and $\forall i,l,$ $\sum_jw_{ij}^{(l)}=1$. 
    \item[\textbf{A2}.] $\X \subseteq \mathbb{R}_+^d$, ie, all learning samples are non-negative vectors.
    \item[\textbf{A3}.] $\forall l\geq 2$, $n_l \geq C$, ie, all hidden layers are wider than the output layer.
    \item[\textbf{A4}.] The loss can be written as: 
    $$\textnormal{loss}(W) = \sum_k \sum_j \ell(o_k^j,y_k^j)$$ with $o_k^j$ the value of the $k^{th}$ output of the DNN for the instance $x^j$.
\end{enumerate}
Regarding \textbf{A1}, one should note that the normalization of the weights have been commonly used both to study DNNs' properties and even to accelerate their training (see \cite{salimans2016}). On the other hand, the non-negativity constraint, which at first glance might seem too restrictive, have also been used by \cite{gautier2016globally} where the authors empirically demonstrate the lack of its negative impact on the NNs' expressiveness.
\textbf{A2} is naturally satisfied by numerous real-world data sets used to train DNNs, such as image collections or text corpora. \textbf{A3} implies for the hidden layers to be wider than the output layer and was used in \cite{NguyenH17} under the name of pyramidal structure assumption. While the importance of depth is often required for DNNs to have good approximation properties, the width of DNNs should be constrained to be wide enough to achieve disconnected decision regions as shown in \cite{NguyenM018}. Finally, \textbf{A4} restricts us to consider only those losses that are computed output-wise thus including many popular norm-based loss function. Note that considering this assumption is less restrictive than many other previous works on the subject that explicitly analyze only the least square loss. 
\begin{figure*}
\begin{minipage}{.4\linewidth}
\resizebox{\columnwidth}{!}{
\begin{tikzpicture}[shorten >=1pt,->,draw=black!50, scale = 0.9]
    \tikzstyle{every pin edge}=[<-,draw=black!50,shorten <=1pt]
    \tikzstyle{neuron}=[circle,fill=black!25,minimum size=24pt,inner sep=0pt]
    \tikzstyle{input neuron}=[neuron, fill=green!20];
    \tikzstyle{output neuron}=[neuron, fill=red!20];
    \tikzstyle{hidden neuron}=[neuron, fill=blue!20];
    \tikzstyle{annot} = [text width=4.8em, text centered]
	\begin{pgfonlayer}{nodelayer}
		\node [input neuron,pin=left: \color{gray}$i_1$] (0) at (-4, 0.75) {};
		\node [input neuron,pin=left: \color{gray}$i_2$] (1) at (-4, -1.25) {};
		\node [hidden neuron] (2) at (0, -0.25) { $f_{v_1^{(1)}}$};
		\node [hidden neuron] (3) at (0, 1.75) {$f_{v_2^{(1)}}$};
		\node [hidden neuron] (4) at (0, -2.25) {$f_{v_3^{(1)}}$};
		\node [output neuron,pin=right: \color{gray}$o_1$] (5) at (4, 1.75) {$f_{v_1^{(2)}}$};
		\node [output neuron,pin=right: \color{gray}$o_2$] (6) at (4, -0.25) {$f_{v_2^{(2)}}$};
		\node [output neuron,pin=right: \color{gray}$o_3$] (7) at (4, -2.25) {$f_{v_3^{(2)}}$};

	\end{pgfonlayer}
	\begin{pgfonlayer}{edgelayer}
		\draw [style=newt] (0) to node{} (3);
		\draw [style=newt] (0) to (2);
		\draw [style=newt] (0) to (4);
		\draw [style=newt] (1) to (3);
		\draw [style=newt] (1) to (2);
		\draw [style=newt] (1) to (4);
		\draw [style=newt] (3) to (5);
		\draw [style=newt] (3) to (6);
		\draw [style=newt] (3) to (7);
		\draw [style=newt] (4) to (7);
		\draw [style=newt] (4) to (6);
		\draw [style=newt] (4) to (5);
		\draw [style=newt] (2) to (7);
		\draw [style=newt] (2) to (6);
		\draw [style=newt] (2) to (5);
	\end{pgfonlayer}
	\end{tikzpicture}
}
\end{minipage}
\hspace{.8cm}
\begin{minipage}{.6\linewidth}
\resizebox{.9\columnwidth}{!}{
\begin{tikzpicture}[scale=0.8]
    \tikzstyle{every pin edge}=[<-,draw=black!50,shorten <=1pt]
    \tikzstyle{neuron}=[circle,fill=black!25,minimum size=24pt,inner sep=0pt]
    \tikzstyle{input neuron}=[neuron, fill=green!20];
    \tikzstyle{output neuron}=[neuron, fill=red!20];
    \tikzstyle{hidden neuron}=[neuron, fill=blue!20];
    \tikzstyle{annot} = [text width=4.8em, text centered]
    \tikzstyle{legend_vert}=[annot, text = olive, minimum size=2.5em]
    \tikzstyle{legend_orange}=[annot, text = violet]
    \tikzstyle{legend_rouge}=[annot, text = red];
	\begin{pgfonlayer}{nodelayer}
		\node [style=basic] (0) at (-6, 2) {};
		\node [style=basic] (1) at (-6, -2) {};
		\node [style=basic] (2) at (-4, -0) {};
		\node [style=basic] (3) at (-4, 2) {};
		\node [style=basic] (4) at (-4, -2) {};
		\node [style=basic] (5) at (1, 2) {};
		\node [style=basic] (6) at (1, -0) {};
		\node [style=basic] (7) at (1, -2) {};
		\node [style=basic] (8) at (-6, -0) {};
		\node [style=basic, fill = olive] (9) at (-9, 1) {};
		\node [style = basic, fill = violet] (10) at (-9, -1) {};
		\node [style=basic] (12) at (-1, 2) {};
		\node [style=basic] (13) at (-1, -0) {};
		\node [style=basic] (14) at (-1, -2) {};
		\node [style=basic2] (18) at (3, 1) {};
		\node [style=basic2] (19) at (3, -0) {};
		\node [style=basic2] (20) at (3, -1) {};
		\node [style=newstyle] (21) at (5, -0) {F};
		\node [style = legend_vert] (23) at (-5, -3) {$\text{Population 1}$ \\ $n_1=1$};
		\node [style = legend_orange] (24) at (0, -3) {$\text{Population 2}$ \\ $n_2=1$};
		\node[style = legend_vert,above of= 9, node distance=0.5cm] (25) {$d_1$};
		\node[style = legend_orange,above of= 10, node distance=0.5cm] (26) {$d_2$};
		\node [style = legend_rouge] (22) at (2,1.9) {$e_1^1$};
		\node [style = legend_rouge] (23) at (4,0.9) {$e_1^2$};
	\end{pgfonlayer}
	\begin{pgfonlayer}{edgelayer}
		\draw [style=pl] (0) to (3);
		\draw [style=pl] (8) to (2);
		\draw [style=pl] (1) to (4);
		\draw [style=newt] (9) to (0);
		\draw [style=newt] (9) to (8);
		\draw [style=newt] (9) to (1);
		\draw [style=newt] (10) to (0);
		\draw [style=newt] (10) to (8);
		\draw [style=newt] (10) to (1);
		\draw [style=newt] (4) to (14);
		\draw [style=newt] (2) to (13);
		\draw [style=newt] (3) to (12);
		\draw [style=newt] (2) to (12);
		\draw [style=newt] (2) to (14);
		\draw [style=newt] (4) to (13);
		\draw [style=newt] (4) to (12);
		\draw [style=newt] (3) to (13);
		\draw [style=newt] (3) to (14);
		\draw [style=pl] (12) to (5);
		\draw [style=pl] (13) to (6);
		\draw [style=pl] (14) to (7);
		\draw [style=mo] (5) to (18);
		\draw [style=mo] (6) to (19);
		\draw [style=mo] (7) to (20);
		\draw [style=mo] (18) to (21);
		\draw [style=mo] (19) to (21);
		\draw [style=mo] (20) to (21);
	\end{pgfonlayer}
\end{tikzpicture}
}
\end{minipage}
\vspace{-.2cm}
\end{figure*}
\begin{figure*}[!t]
\begin{tikzpicture}[trim left={(0,0)}, scale=0.8]
    \draw[-{Latex[width=3mm]}] (3,-3.5)--(3,-4.5) node[midway, right]{Step 1, 2};
    \draw[-{Latex[width=3mm]}] (15,-4.5)--(15,-3.5) node[midway, right]{Step 4};
\end{tikzpicture}
\end{figure*}
\vspace{-.2cm}
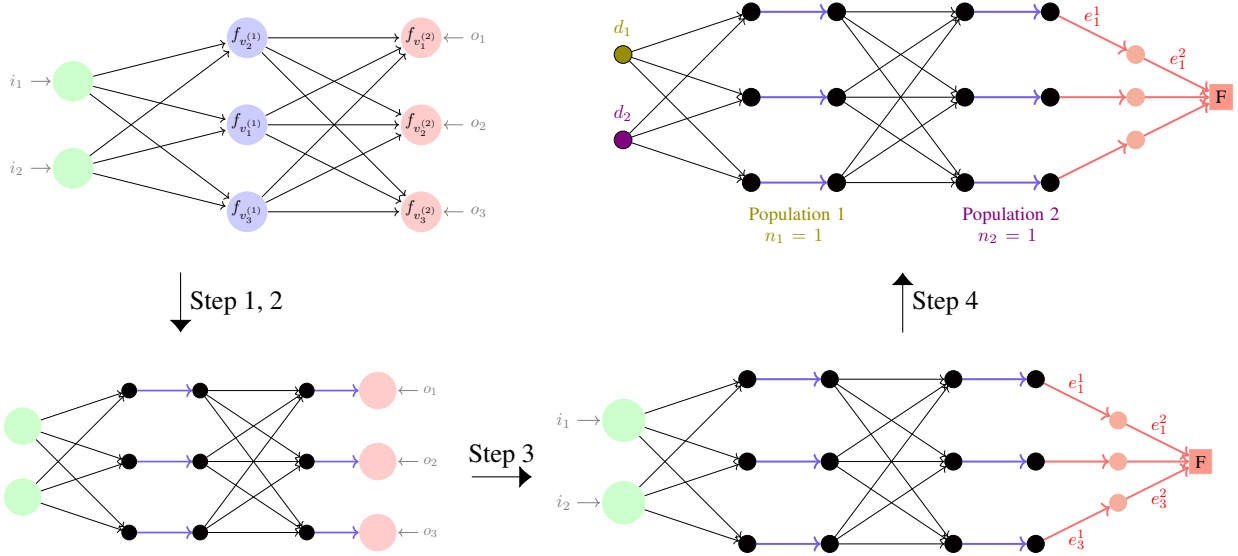
\begin{figure*}[!t]
\begin{minipage}{.4\linewidth}
\resizebox{.9\columnwidth}{!}{
\begin{tikzpicture}[scale = 0.8]
    \tikzstyle{every pin edge}=[<-,draw=black!50,shorten <=1pt]
    \tikzstyle{neuron}=[circle,fill=black!25,minimum size=24pt,inner sep=0pt]
    \tikzstyle{input neuron}=[neuron, fill=green!20];
    \tikzstyle{output neuron}=[neuron, fill=red!20];
    \tikzstyle{hidden neuron}=[neuron, fill=blue!20];
    \tikzstyle{annot} = [text width=4.8em, text centered]
	\begin{pgfonlayer}{nodelayer}
		\node [style=basic] (0) at (-6, 2) {};
		\node [style=basic] (1) at (-6, -2) {};
		\node [style=basic] (2) at (-4, -0) {};
		\node [style=basic] (3) at (-4, 2) {};
		\node [style=basic] (4) at (-4, -2) {};
		\node [style=basic] (8) at (-6, -0) {};
		\node [input neuron] (9) at (-9, 1) {};
		\node [input neuron] (10) at (-9, -1) {};
		\node [style=basic] (12) at (-1, 2) {};
		\node [style=basic] (13) at (-1, -0) {};
		\node [style=basic] (14) at (-1, -2) {};
		\node [output neuron,pin=right: \color{gray}$o_1$] (15) at (1, 2) {};
		\node [output neuron,pin=right: \color{gray}$o_2$] (16) at (1, 0) {};
		\node [output neuron,pin=right: \color{gray}$o_3$] (17) at (1, -2) {};
	\end{pgfonlayer}
	\begin{pgfonlayer}{edgelayer}
		\draw [style=pl] (0) to (3);
		\draw [style=pl] (8) to (2);
		\draw [style=pl] (1) to (4);
		\draw [style=newt] (9) to (0);
		\draw [style=newt] (9) to (8);
		\draw [style=newt] (9) to (1);
		\draw [style=newt] (10) to (0);
		\draw [style=newt] (10) to (8);
		\draw [style=newt] (10) to (1);
		\draw [style=newt] (4) to (14);
		\draw [style=newt] (2) to (13);
		\draw [style=newt] (3) to (12);
		\draw [style=newt] (2) to (12);
		\draw [style=newt] (2) to (14);
		\draw [style=newt] (4) to (13);
		\draw [style=newt] (4) to (12);
		\draw [style=newt] (3) to (13);
		\draw [style=newt] (3) to (14);
		\draw [style=pl] (12) to (15);
		\draw [style=pl] (13) to (16);
		\draw [style=pl] (14) to (17);
	\end{pgfonlayer}
\end{tikzpicture}
}
\end{minipage}
\begin{minipage}{.1\linewidth}
\begin{tikzpicture}[trim left={(0,0)},, scale = 0.8]
\draw[-{Latex[width=3mm]}] (-0.5,0)--(0.5,0) node[midway, above]{Step 3};
\end{tikzpicture}
\end{minipage}
\hspace{-1.2cm}
\begin{minipage}{.6\linewidth}
\resizebox{.9\columnwidth}{!}{
\begin{tikzpicture}[scale = 0.8]
    \tikzstyle{every pin edge}=[<-,draw=black!50,shorten <=1pt]
    \tikzstyle{neuron}=[circle,fill=black!25,minimum size=24pt,inner sep=0pt]
    \tikzstyle{input neuron}=[neuron, fill=green!20];
    \tikzstyle{output neuron}=[neuron, fill=red!20];
    \tikzstyle{hidden neuron}=[neuron, fill=blue!20];
    \tikzstyle{annot} = [text width=4.8em, text centered]
        \tikzstyle{legend_rouge}=[annot, text = red];
	\begin{pgfonlayer}{nodelayer}
		\node [style=basic] (0) at (-6, 2) {};
		\node [style=basic] (1) at (-6, -2) {};
		\node [style=basic] (2) at (-4, -0) {};
		\node [style=basic] (3) at (-4, 2) {};
		\node [style=basic] (4) at (-4, -2) {};
		\node [style=basic] (5) at (1, 2) {};
		\node [style=basic] (6) at (1, -0) {};
		\node [style=basic] (7) at (1, -2) {};
		\node [style=basic] (8) at (-6, -0) {};
		\node [input neuron,pin=left: \color{gray}$i_1$] (9) at (-9, 1) {};
		\node [input neuron,pin=left: \color{gray}$i_2$] (10) at (-9, -1) {};
		\node [style=basic] (12) at (-1, 2) {};
		\node [style=basic] (13) at (-1, -0) {};
		\node [style=basic] (14) at (-1, -2) {};
		\node [style=basic2] (18) at (3, 1) {};
		\node [style=basic2] (19) at (3, -0) {};
		\node [style=basic2] (20) at (3, -1) {};
		\node [style=newstyle] (21) at (5, -0) {F};
		\node [style = legend_rouge] (22) at (2,1.9) {$e_1^1$};
		\node [style = legend_rouge] (23) at (4,0.9) {$e_1^2$};
		\node [style = legend_rouge] (24) at (2,-1.9) {$e_3^1$};
		\node [style = legend_rouge] (25) at (4,-0.9) {$e_3^2$};
	\end{pgfonlayer}
	\begin{pgfonlayer}{edgelayer}
		\draw [style=pl] (0) to (3);
		\draw [style=pl] (8) to (2);
		\draw [style=pl] (1) to (4);
		\draw [style=newt] (9) to (0);
		\draw [style=newt] (9) to (8);
		\draw [style=newt] (9) to (1);
		\draw [style=newt] (10) to (0);
		\draw [style=newt] (10) to (8);
		\draw [style=newt] (10) to (1);
		\draw [style=newt] (4) to (14);
		\draw [style=newt] (2) to (13);
		\draw [style=newt] (3) to (12);
		\draw [style=newt] (2) to (12);
		\draw [style=newt] (2) to (14);
		\draw [style=newt] (4) to (13);
		\draw [style=newt] (4) to (12);
		\draw [style=newt] (3) to (13);
		\draw [style=newt] (3) to (14);
		\draw [style=pl] (12) to (5);
		\draw [style=pl] (13) to (6);
		\draw [style=pl] (14) to (7);
		\draw [style=mo] (5) to (18);
		\draw [style=mo] (6) to (19);
		\draw [style=mo] (7) to (20);
		\draw [style=mo] (18) to (21);
		\draw [style=mo] (19) to (21);
		\draw [style=mo] (20) to (21);
	\end{pgfonlayer}
\end{tikzpicture}
}
\end{minipage}
\caption{Illustration of how a non-atomic congestion game is constructed from a given DNN. \textbf{(upper left)} Example of a DNN with $d=2$, $L=1$, $n_2 = 3$ and $C=3$; \textbf{(upper right)} Graph associated to the DNN presented on the left. $B$ is the set of black edges, $J$ is the set of blue edges and $T$ is the set of red edges. \textbf{(bottom left)} and \textbf{(bottom right)} are the intermediate steps.}
\label{fig:NCG_lemma1}
\end{figure*}
\subsection{Analysis of linear DNNs} 
We start by proving a first result related to linear DNNs. We recall that while being quite restrictive in practice, the theoretical analysis of this setting is still challenging as it represents a non-convex optimization problem. Furthermore, we use it as a cornerstone result for our future developments as it allows to illustrate our proposed construction.
\begin{lemma}
Assume A1-4, let DNN be defined as $N = (V, E, I, O, F)$ with $F = \{f: \forall z, f(z)=z\}$, let $\textnormal{loss}(\cdot)$ be its associated loss function and let $\mathcal{L}=\{(x^j,y^j)\}_{j=1}^M$ be the learning sample. Then, one can construct a non-atomic congestion game $\text{\textnormal{NCG}}_N^\textnormal{loss} = (E,c,S,n,a)$ fully defined in terms of $N$, $\textnormal{loss}(\cdot)$ and $\mathcal{L}$.
\label{lemma_1}
\end{lemma}
\begin{proof}
We start by defining $E$ of the corresponding congestion game by applying a set of predefined rules to the network $N$ as follows: 
\begin{enumerate}
    \item Each edge of $N$ becomes an edge of $\text{\textnormal{NCG}}_N^\textnormal{loss}$. Denote by $B$ the set of these edges.
    \item Each node of $N$ with an activation function (nodes of the hidden layers and of the output layer) becomes an edge of the $\text{\textnormal{NCG}}_N^\textnormal{loss}$. Denote by $J$ the set of these edges.
    \item Each node of the output layer becomes a concatenation of $M$ edges of $\text{\textnormal{NCG}}_N^\textnormal{loss}$ which are added to the edge of $J$ so that the last edge points to a common node of the congestion game $F$. Denote by $T$ the set of these edges such that $T = \{e_k^j: e_k^j$ associated to $o_k$ for a tuple $(x^j,y^j)\}$. The index $k$ is used for the output number which is considered $1 \leq k \leq C$ and for a fixed $k$ the index $j$ is the number of the arc which is considered $1 \leq j \leq M$. Let $p_k$ be the set of the paths of the neural network (seen as a DAG) that include the concatenation of the $M$ edges associated to output $k$. 
    \item The $i^{th}$  node of the input layer become a node of the congestion game named $d_i$ for $1 \leq i \leq d$. 
\end{enumerate}
An illustrative example of such transformation is given in Figure \ref{fig:NCG_lemma1}. We now define $S$, $n$, $c$ and $a$ as follows. 
\begin{enumerate}[leftmargin=1.5cm]
    \item[\textbf{S,n}.] One population of players $i$ is created for each node $i$ of the input layer. The set of strategies of the player $i$, $S_i$, is the set of the paths from $d_i$ to $F$.  The size of the population $i$ is $1$ for each $i$, ie, $n_i = 1$.
    \item[\textbf{c}.] We define the cost of the edges as follows: 
    \[c_e(\xi) = \left\{  \begin{array}{l l} 0, & \text{ if } e \in B \cup J\\
     										\ell (\xi,y_k^j)/\xi, & \text{ if } e=e_k^j \in T. \end{array}\right.\]
    \item[\textbf{a}.] We define the rate of consumption as follows:
    \[\forall S \in S_i,\ a_{S,e} = \left\{  \begin{array}{l l} 1, &  \text{ if } e \in B \cup J\\
     										x_i^j, & \text{ if }  e=e_k^j \in T. \end{array}\right.\]
  
\end{enumerate}
The congestion game is entirely defined.
\end{proof}

\begin{remark} We implicitly assume that $c_{e_k^j}(\xi)=\ell (\xi,y_k^j)/\xi$ is positive and increasing in $\xi$ so as to respect the positivity and the increase of the cost functions of the associated congestion game. Moreover, it is assumed that $\ell (\xi,y_k^j)/\xi$ is well defined in 0. We discuss later how these assumption can be shown to cover some loss functions used in practice. 
\end{remark}
Let us now give the main theorem of the paper.
\begin{theorem}
Under the assumptions of Lemma \ref{lemma_1}, let $\ell(\xi,y_k^j) = A_k^j \xi^{\beta}$ with $A_k^j > 0$, ${\beta \geq 2}$. Then, given a neural network $N$, every local minimum of a loss function $\textnormal{loss}(\cdot)$ associated to $N$ is a Wardrop equilibrium of the associated congestion game $\textnormal{NCG}^\textnormal{loss}_N$. 
\label{theorem_1}
\end{theorem}
\begin{proof_idea}
In what follows, we present the sketches of the proofs when they are lengthy, and refer the interested reader to the \hyperlink{appendix}{Appendix} for their complete versions. We start by relating the weight $W$ of the neural network $N$ to the flow in the associated congestion game such that the loss of the neural network becomes equal to the social cost of the associated congestion game. Then, we show how a local minimum $W$ of the loss function induces a distribution $z_W$ which is a Wardrop equilibrium of the associated congestion game. To this end, we use the result from \cite{Kinderlehrer2000} showing that every local minimum $x^*$ of a function $h$ belonging to class $C^1$ and defined on a closed and convex subset $X \subseteq \mathbb{R}^n$ verifies the following variational inequality: 
$$\langle \nabla h(x^*) , x-x^* \rangle \geq 0, \;\; \forall x \in X.$$
On the other hand, we can characterize the Wardrop equilibrium of a non-atomic congestion game  by proving that a distribution $z^*$ is a Wardrop equilibrium of $\textnormal{NCG}_N^\textnormal{loss}$ if and only if:
$$ \sum_i \sum_k \sum_j x_i^j c_k^j(\tilde{z}_{k}^{j*})(z_{k,i} - z_{k,i}^*) \geq 0, \;\; \forall z \in Z,$$
where $z_{k,i} = \sum_{S \in S_i \cap p_k} z_S$ with $p_k$ being the set of paths which include $(e_k^{j'})_{1 \leq j' \leq M}$ and $\tilde{z}_{k}^j = \sum_i x_i^j z_{k,i}.$
For the considered family of loss functions, we further show that variational characterization of a local minimum and that of a Wardrop equilibrium imply one another. The desired result is obtained by establishing that the flow associated to a local minimum $W$ is a Wardrop equilibrium of the associated game.
\end{proof_idea}
This theorem is important as it allows to deduce two corollaries, one on neural networks and the second about congestion games.
\begin{corollary}
Under the assumptions of Theorem 1, every local minimum of $\textnormal{loss}(\cdot)$ is a global optimum.
\end{corollary}
\begin{proof}
For a loss of this type, we have shown that a local minimum of the linear neural network is a Wardrop equilibrium of the associated congestion game. As a local minimum, the global minimum is also a Wardrop equilibrium. It is known that for non-atomic congestion games, the Wardrop equilibria have the same social cost. Because the value of the loss function at $W$ is equal to the value of the social cost of the associated congestion game at $z_W$, local minimums and global ones have the same value which is the value of the Wardrop equilibria. 
\end{proof}
\begin{corollary}
Under the assumptions of Theorem 1, $\textnormal{PoA}(\textnormal{NCG}_N^\textnormal{loss}) = 1$.
\end{corollary}
\begin{proof}
A global minimum of the loss of $N$ is a social optimum of the associated congestion game. Global minima are also Wardrop equilibria. Then, because the value of the loss function at $W$ is equal to the value of the social cost of the associated congestion game at $z_W$, we get $\text{WE}=\text{SO}$ and $\text{PoA}=1$.
\end{proof}

\subsection{Learning with loss function from Theorem 1}
We now explain how we can use the loss studied in Theorem 1 in practice when dealing with classification task where $y^j$ is a binary vector with only one coordinate equal to $1$ and the rest being equal to $0$. For each $j$, let us denote by $\mathsf{e}_j$ the coordinate of the vector $y^j$ which is equal to 1, \ie, $y^j_{\mathsf{e}_j} =1$ and $y^j_k = 0$ for $k \neq \mathsf{e}_j$. Moreover, we consider normalized inputs such that $\sum_i x_i^j = 1$ for all $j$. Then, we can use our loss functions in the following way: 
\begin{enumerate}
    \item We fix $\beta \geq 2$.
    \item For each $j$, we impose $A_k^j =1$ for $k \neq \mathsf{e}_j$ and $A_k^j =0$ if $k = \mathsf{e}_j$. By this way, we penalize the outputs of the network which have to be equal to 0, while putting no penalty on the outputs that need to be equal to 1. 
\end{enumerate}
We can deduce the following corollary. 
\begin{corollary}
Under the assumptions of Theorem 1, let $C=2$ and let $\ell$ be the squared loss. Then, a local minimum of $\textnormal{loss}(\cdot)$ is a global minimum.
\label{corollary:binary}
\end{corollary}
\begin{proof_idea}
We rewrite our loss and the squared loss and show that for $C>2$ they differ by a constant. We then analyze this constant and prove that for $C=2$ it reduces to the squared loss thus leading to the same optimization objective. 
\end{proof_idea}

\subsection{Extension to DNNs with ReLUs} 
We now proceed to a study of non-linear DNNs with activation functions $F$ given by ReLUs defined as: 
\begin{align}\label{eq:f_nonlinear}
\begin{array}{l|rcl}
f_v : & \R & \longrightarrow & \R \\
    & x & \longmapsto & \max(0,x).
    \end{array}  
\end{align}
Such a non-linear DNN can be seen as a linear DNN where some paths of the graph underlying the network fail and thus can be also seen as non-atomic congestion game where edges of $J$, which represent the activation functions, can fail depending on the congestion that occurs on them. More precisely, one can show that for a non-linear DNN with ReLUs, its $k^{th}$ output for the instance $x^j$ is given by:
$$o_k^j = \sum_{p \in p_k} z^j_p x^j_p  w_p,$$
where $p_k$ is the set of paths that end to output $k$,  $w_p$ is the product of the weights on the path $p$ and $x^j_p$ is the value of the coordinate of $x^j$ from which the path $p$ starts. As for $z^j_p$, it is a variable that is equal to $1$ if all ReLUs $f_v$ encountered in the path $p$ are such that $f_v(g_v^j) = g_v^j$ where $g_v$ is the value of the node $v$ on the example $x^j$ and $0$, otherwise. In other words, the variable $z^j_p$ reflects whether the path $p$ is active ($z^j_p = 1$) or not ($z^j_p = 0$) depending on the ReLU activation on the path $p$ for the instance $x^j$. 

As non-linear DNNs are notoriously hard to study, most papers introduce simplifications to model non-linearities in order to analyze the simplified models afterwards. One prominent example of such modelization was introduced by \cite{choromanska2014loss} and further improved by \cite{kawaguchi2016deep} who successfully discarded several of the unrealistic assumptions of the original model and lightened others. For our analysis, we use some of the lighter assumptions made in the latter paper. The first assumption, denoted by $A1p_m$ in the corresponding paper, states that $z^j_p$ are Bernoulli random variables with the same probability of success $\rho$. The second assumption, called $A5u_m$, states that $z^j_p$ are independent from the inputs $\{x^j\}_{j=1}^M$ and the weight parameters $W$. These assumptions, although remaining unrealistic in case of $A5u_m$, allow us to write the expected output $o_k^j$ as follows: 
\begin{align}
\label{eq:output_nonlinear}
  \mathbb{E}(o_k^j) = \sum_{p \in p_k} \rho x^j_p  w_p = \rho \sum_{p \in p_k} x^j_p  w_p.
 \end{align}
One can remark that the output of the network has simply been multiplied by $\rho$.
Given a non-linear DNN $N$, let $\text{lin}(N)$ be the linear DNN associated to $N$ where all activation functions are replaced by the function:
\begin{align*}
\begin{array}{l|rcl}
f_v : & \R & \longrightarrow & \R \\
    & x & \longmapsto & x.
    \end{array}
\end{align*}
We now show how we can reduce non-linear DNNs to congestion games with failures by adapting the results obtained for atomic congestion games with failures in the paper \cite{article} to a non-atomic case. 
\begin{lemma}
Assume A1-4, $A1p_m$, $A5u_m$, let $N = (V, E, I, O, F)$ with $O$ and $F$ defined as in \eqref{eq:output_nonlinear} and \eqref{eq:f_nonlinear}, respectively. Let $\textnormal{loss}(\cdot)$ be its associated loss function and let $\{(x^j,y^j)\}_{j=1}^M$ be the available learning sample. Then, $N$ can be reduced to a non-atomic congestion game with failures $\text{\textnormal{NCGF}}_N^\textnormal{loss} = (E,c,S,n,a)$ fully defined in terms of $N$ and $\textnormal{loss}(\cdot)$.
\label{lemma_1_nonlinear}
\end{lemma}
\begin{proof}
$E,S,n,a$ remain the same as for $\text{\textnormal{NCG}}_N^\textnormal{loss}$ from Lemma \ref{lemma_1}. The only modification is that if a player chooses a path with no failures, then its cost is $c_S(z) = \sum_{e \in S} a_{S,e} c_e(z_e)$ where $z$ is the flow of the game such that failures are taken into account. Otherwise, we impose $c_S(z) = w_i$ where $w_i$ is a constant associated to a player type $i$. $\text{\textnormal{NCGF}}_N^\textnormal{loss}$ is now defined. 
\end{proof}
Given $W$ and our set of assumptions, we study the same loss as in the linear case but on the expected outputs of the neural networks, ie, 
\begin{align}\label{eq:loss_nonlinear}
    \textnormal{loss}(W) = \sum_j \sum_k \ell(\mathbb{E}(o_k^j),y_k^j),
\end{align}
where $o_k^j$ is the $k^{th}$ output of $N$ for instance $j$.
We further let $\beta = 2$ in the definition of $\ell$ as done in \cite{kawaguchi2016deep} for the squared loss. We now establish the equivalence between the local minima of the non-linear model to those of the linear one.
\begin{lemma}
Under the assumptions of Lemma \ref{lemma_1_nonlinear}, let $\beta=2$ and let $\ell$ be as in \eqref{eq:loss_nonlinear}. Then, a local minimum of $\textnormal{loss}(W)$ of $N$ is a local minimum of the loss function of the corresponding linear network from Lemma \ref{lemma_1}.
\label{lemma_2_nonlinear}
\end{lemma}
The final result obtained through the transitivity of properties shown in Theorem 1 is stated as follows.
\begin{corollary}
Under the assumptions of Lemma \ref{lemma_2_nonlinear}, a local minimum of a loss function in  \eqref{eq:loss_nonlinear} is a Wardrop equilibrium of the game associated to the linear network from Lemma \ref{lemma_1}. 
\end{corollary}
In particular, this corollary suggests that the model introduced for non-linear DNNs in \cite{choromanska2014loss} is equivalent to studying a linear DNN. 

Overall, our contributions establish the state-of-the-art results related to the analysis of the loss surface of both linear and non-linear DNNs following a completely novel approach. While this is an important contribution in itself, we belive that it further paves the way to several other highly promising research directions presented below. 

\section{Open problems}
In the introduction we claimed that our analysis can be used as a tool to prove other fundamental results about DNNs by modelling them as congestion games. Below, we aim to highlight this claim by rigorously formulating three open problems for future research. 
\paragraph{Impact of DNNs architecture} Characterizing PoA depending on the network topology and the used cost function is commonly done in the field of congestion games and we expect it to be very useful for DNNs as well. Indeed, it is known that PoA of non-atomic congestion games is independent of network topology \cite{Colini-Baldeschi17} when cost functions are polynomials of an arbitrary degree. Consequently, one may wonder whether our results can be extended beyond multi-layer networks to take into account other network architectures, such as U-Nets \cite{RonnebergerFB15}, with potentially different activation and/or loss functions. More formally, we propose the following open problem.
\begin{op}
For a DNN $N = (V, E, I, O, F)$, let $\textnormal{NCG}$ be a congestion game such that every local minimum/critical point and global minima of $\textnormal{loss}(\cdot)$ associated to $N$ are Wardrop equilibrium and social optimums of $\textnormal{NCG}$, respectively. Then, $\textnormal{PoA}(\textnormal{NCG}) = 1$ and does not depend on $G=(V,E)$, $F$ and $\textnormal{loss}(\cdot)$.
\end{op}
We note that this open problem can also lead to a negative result where the local minima inefficiency can be proved to be arbitrary high. To obtain such result, one will have to consider \textit{atomic} congestion games $\textnormal{CG}$ for which, contrary to Wardrop equilibria, different Nash equilibria can lead to outcomes with different costs. If there exists a DNN such that its associated atomic congestion game has $\text{PoA}(\textnormal{CG}) > 1$ for a particular choice of $G$, $F$ and $\text{loss}(\cdot)$, then such neural architectures may exhibit an arbitrary bad behaviour during the optimization. One such example is the $\text{PoA}$ in atomic splittable games with polynomial cost functions of degree $d$ for which there exists particular network topologies with $\text{PoA}$ behaving as $((1+\sqrt{d+1})/2)^{d+1}$ \cite{rough_local}. 

\color{black} \paragraph{Speed of convergence} Apart from the geometric approach that relies on the equivalence between local minima and the global optimum, another way to analyze DNN's behaviour is to study directly the optimization dynamics of SGD and its variations. In the context of games, this latter can be modelled by repeating the one-shot game, \ie, the game with one data point or a mini batch, over $T$ time steps and analyzing the average of the costs associated with the outcomes of each step. In our work, we provided a characterization of the PoA for one-shot non-atomic congestion game, but this analysis can be further applied to repeated games using the extension theorems and the notion of $(\lambda,\mu)-$smoothness developed in \cite{robustness}. In line with what we discussed above, we now consider atomic congestion games to allow for Nash equilibria with different costs. Such games can be characterized by the notion of a robust PoA $\rho$ defined in terms of the parameters $\lambda$ and $\mu$ and coinciding with the traditional PoA in several cases of interest. For these games, extension theorems ensure that the sequence of outcomes of a smooth game where every player experiences vanishing average (external) regret converges to the optimal outcome times the robust PoA.  This is due to the fact that for every pure/mixed/correlated/coarse-correlated equilibrium $z'$, $\frac{\mathbb{E}_{z \sim z'}[\textnormal{SC}(z)]}{\textnormal{SO}(\textnormal{CG})} \leq \rho$. Then, as a sequence of outcomes with vanishing external regret converges to a correlated equilibrium, we have the following.
\begin{op}
For a deep neural network $N = (V, E, I, O, F)$,  let $\textnormal{CG}$ be its corresponding $(\lambda,\mu)-$smooth congestion game such that \color{black}
\begin{enumerate}
    \item $\textnormal{loss}(W)$ is equal to the social cost $\textnormal{SC}(z_W)$.
    \item $\textnormal{SO(CG})$ is the global optimum of $\textnormal{loss}(\cdot)$. 
\end{enumerate}
Then, if $\textnormal{loss}(W^i)$ are social costs associating to a vanishing average
(external) regret, the following holds:
$$\frac{1}{T}\sum_{i=1}^T \textnormal{loss}(W^i) \leq (\rho(\textnormal{CG}) + o(1)) \textnormal{loss}(W^*) \text{ as } T \mapsto \infty$$
where $W^i$ is the outcome of iteration $i$ and $W^*$ is a global optimum.

Moreover, if each critical point $W$ of the loss function is either pure, mixed, correlated or coarse-correlated equilibrium of $\textnormal{CG}$, then $\frac{\textnormal{loss}W}{\textnormal{loss}W^*} \leq \rho(\textnormal{CG})$.
\end{op}
Note that \cite{balduzzi2016deep} related the coarse-correlated equilibrium to critical points of a non-linear DNN, but their arguments were shown to be flawed in \cite{SchuurmansZ16}. This latter work managed to draw the equivalence between critical points and Nash equilibrium correctly and used a popular regret matching algorithm to successfully learn DNNs. Unfortunately, their proposed learning strategy is applied to games solved on each vertex of a DNN and thus is not guaranteed to converge to a globally optimal strategy. In this regard, congestion games offer a more convenient alternative as they benefit from the convergence guarantees and allow to characterize the speed of this convergence based on the characteristics of the considered game.

\paragraph{Beyond backprop} As discussed above, game-theoretical interpretation of DNNs had already led to new learning strategies used to find equilibrium points in the corresponding games. As shown in \cite{SchuurmansZ16}, regret matching algorithm outperforms widely-used SGD and Adam methods and leads to sparser networks with higher accuracy when deployed on the test set. In the context of our work, we would like to make a step further and go beyond the regret matching algorithm mentioned above by proposing a new learning strategy based on optimal transport \cite{villani}. Optimal transport considers a problem of transforming one probability measure into another following the principle of the least effort. While traditionally optimal transport does not account for congestion effects, several recent works studied this variation and showed that in this case the solution can be related to the notion of the equilibria of Wardrop type \cite{carlier08,BlanchetC16}. This leads to the following open problem.
\begin{op}
For a deep neural network $N = (V, E, I, O, F)$, let $\textnormal{NCG}$ be its corresponding congestion game with equilibria given by:
\begin{align}
    \eta^* \in \textnormal{argmin}_\eta W_c(\mu,\eta) + \mathcal{E}(\eta),
    \label{wass_flow:our}
\end{align}
where $W_c$ is the Wasserstein distance between a measures defined on the input flow of $\textnormal{NCG}$ and $\mathcal{E}(\eta)$ is a function of congestion for an action distribution $\eta$. Then, $\eta^*$ is a critical point of $\textnormal{loss}(\cdot)$.
\end{op}
We note that several papers \cite{ChizatB18,rostkoff,mei18} used the notion of the Wasserstein gradient flow to show the convergence of convex optimization methods for overparametrized models. Their underlying idea was to consider a problem of learning a measure that minimizes the DNN's loss function and to study the dynamics of the gradient descent performed on its weights and positions. One question to answer thus is whether the Wasserstein gradient flow of solving \eqref{wass_flow:our} is naturally linked to the Wasserstein gradient flow considered in previous work? Such an equivalence may well indicate that the game-theoretical interpretation of DNNs reconciles both geometric approaches for studying DNNs loss surface and those based on analyzing their optimization dynamics.


\newpage

\Large{\hypertarget{appendix}{\textbf{Appendix}}}%

\normalsize
\section*{Proof of Theorem \ref{theorem_1}}
\begin{proof}
In what follows, we write $c_k^j := c_{e_k^j}$ and similarly for all quantities that have such subscripts to avoid cumbersome notations.
The proof of this theorem relies on several lemmas with the first one showing how the weight matrix $W$ of the neural network $N$ is related to the flow in the associated congestion game such that the \textnormal{loss} of the neural network becomes equal to the social cost of the associated congestion game.
\begin{lem}
Under the assumption of Lemma \ref{lemma_1}, a configuration $W$ of the neural network $N$ defines an action distribution $z_W$ of the associated congestion game $\text{\textnormal{NCG}}_N^{\textnormal{loss}}$ such that $\text{\textnormal{loss}}(W) = \textnormal{SC}(z_W)$. Similarly, for every action distribution $z_W$ of the associated congestion game $\text{\textnormal{NCG}}_N^{\textnormal{loss}}$ there exists a set of weights $W$ of $N$ such that $\text{\textnormal{loss}}(W) = \textnormal{SC}(z_W)$.
\end{lem}
\begin{proof}
($\longrightarrow$) Given a flow $z$, the social cost of a congestion game can be written as:
$$\textnormal{SC}(z) = \sum_{e \in E} c_e(\tilde{z}_e) \tilde{z}_e.$$
For the studied congestion games, the following holds:
\begin{align*}
    \textnormal{SC}(z) &= \sum_{e \in E} c_e(\tilde{z}_e) \tilde{z}_e\\
    &= \sum_k \sum_j c^j_k(\tilde{z}_k^j) \tilde{z}_k^j \\
    &= \sum_k \sum_j \ell(\tilde{z}_k^j ,y_k^j).
\end{align*}
Given a set of weights $W$ of a given network $N$, the \textnormal{loss} of the network can be written as: 
$$\text{\textnormal{loss}}(W) = \sum_j \sum_k \ell(\tilde{b}_k^j,y_k^j)$$
with $b_{k,i} = \sum_{p \in S_i \cap p_k} w_p$ where $w_p$ is the product of the weights encountered in the path $p$ and $p_k$ is the set of paths that include the $M$ edges associated to the output $k$, $\tilde{b}_k^j = \sum_i x_i^j b_{k,i}$.

In order to relate $\tilde{z}_k^j$ and $\tilde{b}_k^j$, we let $k \in [1,\dots,C]$ and note that a player of type $i$ who wishes to travel on the edge $e_k^j$ travels on all the edges $(e_k^{j'})_{1 \leq j' \leq M}$. Thus, the measure of population $i$ using the edge $e_k^j$ is equal to the measure of population $i$ that uses the edge $e_k^{j'}$ and is denoted by $z_{k,i} = \sum_{S \in S_i \cap p_k} z_S$ where $p_k$ is the set of paths which include  $( e_k^{j'})_{1 \leq j' \leq M}$. One can verify that $$\tilde{z}_k^j = \sum_i x_i^j z_{k,i}.$$

Relating $\tilde{z}_k^j$ and $\tilde{b}_k^j$ now boils down to establishing the equality between $z_{k,i}$ and $b_{k,i}$. To this end, we note that for an action distribution $z$ defined such that for each $i$ and $S \in S_i$, $z_S = w_S$ with $w_S$ being the product of the weights encountered in the path $S$, $z_{k,i} =  \sum_{ S \in S_i \cap p_k } z_S  = \sum_{ S \in S_i \cap p_k } w_S = b_{k,i}$ which implies $\tilde{z}_k^j = \tilde{b}_k^j$ so that 
\begin{align*}
    \text{\textnormal{loss}}(W) &= \sum_j \sum_k \ell(\tilde{b}_k^j,y_k^j) = \sum_j \sum_k \ell(\tilde{z}_k^j,y_k^j)\\
    &= \textnormal{SC}(z_W).
\end{align*}
Note that $z$ is a valid distribution since $\sum_{S \in S_i} z_S = 1$ as each subgraph $S_i$ of the network with $d_i$ as root is a probability tree with positive and normalized weights. 
($\longleftarrow$) We prove the second statement of this lemma below. From the Assumption 1, it follows that all the matrices we consider in this proof are positive and verify that the sum of the coefficients on each column is equal to 1. Let $w^{(1)},...,w^{(L)}$ be the matrices of weights of the neural network such that, for $1 \leq \ell \leq L$, $w^{(\ell)}$ is the matrix of the weights which stands between layer $\ell-1$ and layer $\ell$. Concerning the dimensions of the matrices, we have that $w^{(1)} \in {M}_{n_1,d}(\mathbb{R}_+)$ and for $1 \leq \ell \leq L$, $w^{(\ell)}\in {M}_{n_{\ell},n_{\ell-1}}(\mathbb{R}_+)$. Let $w' \in {M}_{C,d}(\mathbb{R}_+)$ be the following matrix: $w' = (z_{k,i})_{1 \leq k \leq C \atop 1 \leq i \leq d}$.
Let $w''$ be the matrix such that $w'' = (b_{k,i})_{1 \leq k \leq C \atop 1 \leq i \leq d} $. One can verify that $w'' = w^{(L)}*...*w^{(1)}$
Now, we see that our problem boils down to whether for each matrix $w'$, there exists $w^{(1)},...,w^{(L)}$ such that $w' = w^{(L)}*...*w^{(1)}$. We will show that the answer is positive by using the fact that each matrix $w^{(\ell)}$ has dimensions superior to $C$ due to the Assumption 3. 
Let $w'$ be a positive and normalized matrix. Let us take $w^{(1)},...,w^{(L)}$ such that: 
    \[
    w^{(1)} = 
    \begin{pmatrix}
     0 \\
     w'
    \end{pmatrix},\quad 
    \text{and for } \ell=2,\dots,L:\quad
    w^{(\ell)} = 
    \begin{pmatrix}
     H_{\ell} & 0 \\
     0 & I_C
    \end{pmatrix}\quad
    \text{with } 
    H_{\ell} = 
    \begin{pmatrix}
     1 & \dots & 1 \\
     0 & \dots & 0 \\
     \vdots &   & \vdots \\
     0 & \dots & 0 
    \end{pmatrix},\
    I_C \in {M}_C(\mathbb{R}).
    \]
    One can verify that $w^{(1)},...,w^{(L)}$ are positive and normalized matrices. Then, we have $w^{(1)},...,w^{(L)}$ admissible matrices such that $w' = w^{(L)}*...*w^{(1)}$. 
\end{proof}
We now proceed by showing how a local minimum $W$ of the \textnormal{loss} function induces a distribution $z_W$ which is a Wardrop equilibrium of the associated congestion game. To this end, we use the result from \cite{Kinderlehrer2000} showing that every local minimum $x^*$ of a function $h$ belonging to class $C^1$ and defined on a closed and convex subset $X \subseteq \mathbb{R}^n$ verifies the following variational inequality: 
$$\langle \nabla h(x^*) , x-x^* \rangle \geq 0, \;\; \forall x \in X.$$
Given $W$, the \textnormal{loss} of the network only depends on $\{b_{k,i}\}$ (see lemma above) which induces the outputs $\{\tilde{b}_k^j\}$. Then, we can define the \textnormal{loss} on the set of admissible families $B = \{b_{k,i}\}$ and verify that $B$ is convex and closed. Given the particular shape of the studied \textnormal{loss} functions, one can also verify that the \textnormal{loss} is $C^1$ on a neighborhood of $B$. Then, we get that every local minimum $b^*$ of the \textnormal{loss} function on $B$ verifies the variational inequality: 
$$\langle \nabla \text{\textnormal{loss}}(b^*) , b-b^* \rangle \geq 0, \;\; \forall b \in B.$$
We can now prove the following lemma.
\setcounter{lemma}{2}
\begin{lem}
For any $b$ and $b^*$ in $B$, the following holds: 
\begin{align*}
    \langle \nabla \text{\textnormal{loss}}(b^*) , b-b^* \rangle &= \sum_i \sum_k \sum_j x_i^j c_k^{j\prime}(\tilde{b}_k^{j*}) \tilde{b}_k^{j*} (b_{k,i} - b_{k,i}^*)\\
    &+ \sum_i \sum_k \sum_j x_i^j c_k^j(\tilde{b}_k^{j*})(b_{k,i} - b_{k,i}^*).
\end{align*}
\end{lem}
\begin{proof}
Be $b \in B$, we have:
\begin{align*}
    \textnormal{loss}(b) &= \sum _k \sum_j \ell(\tilde{b}_k^{j}, y_{k}^j ) \\
    &=\sum _k \sum_j c_k^j(\tilde{b}_k^{j}) \tilde{b}_k^{j}
\end{align*}
with $\tilde{b}_k^{j} = \sum_i x_i^j b_{k,i}$. 
For $b \in B$ we can compute that: 
\begin{align*}
    \frac{\partial \textnormal{loss}}{\partial b_{k,i}}(b) &= \sum_j  \frac{\partial c_k^j}{\partial b_{k,i}} (\tilde{b}_k^{j}) \tilde{b}_k^{j} + c_k^j(\tilde{b}_k^{j})  \frac{\partial \tilde{b}_k^{j}}{\partial b_{k,i}} \\
    &= \sum_j c_k^{j\prime}(\tilde{b}_k^{j}) x_i^j \tilde{b}_k^{j} + c_k^j(\tilde{b}_k^{j}) x_i^j \\
    &= \sum_j c_k^{j\prime}(\tilde{b}_k^{j}) x_i^j \tilde{b}_k^{j} + \sum_j c_k^j(\tilde{b}_k^{j}) x_i^j.
\end{align*}
For $b \in B$ and $b^* \in B$, the previous calculations lead to the desired result.
\begin{align*}
    \langle \nabla \textnormal{loss}(b^*), b-b^* \rangle & = \sum_i \sum_k \frac{\partial \textnormal{loss} }{\partial b_{k,i}}(b^*) (b_{k,i} - b_{k,i}^*) \\
    &= \sum_i \sum_k \sum_j x_i^j c_k^{j\prime} (\tilde{b}_k^{j*}) \tilde{b}_k^{j*} (b_{k,i} - b_{k,i}^*) + \sum_i \sum_k \sum_j x_i^j c_k^j(\tilde{b}_k^{j*})(b_{k,i} - b_{k,i}^*).
\end{align*}
\end{proof}
On the other hand, we can characterize the Wardrop equilibrium of a non-atomic congestion game using the following variational inequality.
Let $Z$ be the set of admissible flows for $\textnormal{NCG}_N^{\textnormal{loss}}$.
\begin{lem}
A distribution $z^*$ is a Wardrop equilibrium of $\textnormal{NCG}_N^{\textnormal{loss}}$ if and only if:
$$ \sum_i \sum_k \sum_j x_i^j c_k^j(\tilde{z}_k^{j*})(z_{k,i} - z_{k,i}^*) \geq 0, \;\; \forall z \in Z.$$ 
\end{lem}
\begin{proof}
($\longrightarrow$) In the following, we will use \cite[Proposition 2.7]{nacg_roughgarden}. It states that if $z$ is a Wardrop equilibrium, then for each player type $i$ there is a real number $c_i(z)$ such that all strategies $S \in S_i$ with $z_S >0 $ verify $c_S(z) = c_i(z)$. Then, the social cost of $z$ is: 
$$\textnormal{SC}(z) = \sum_{i=1}^d c_i(z) n_i.$$
We can add that if $z$ is a Wardrop equilibrium, then for each player type $i$ and strategy $S \in S_i$, if $z_S = 0$ then $c_S(z) \geq c_i(z)$. This addition is immediate from the definition of the Wardrop equilibrium because it would be in the interest of the players of type $i$ to use the strategy $S$ otherwise.

We will now apply this proposition to our case. Let $z^*$ be a Wardrop equilibrum of $\text{NCG}_N^{\textnormal{loss}}$. 
For each player type $i$, there exists a number $c_i(z^*)$ such that for all $k$, if $z^*_{k,i} > 0$ then $$\forall S \in S_i \cap p_k,  c_S(z^*) = \sum_j x_i^j c_k^j(\tilde{z}_k^{j*}) = c_i(z^*).$$ 
Moreover, if  $z^*_{k,i} = 0$ then for all $S \in S_i \cap p_k$, we have $$c_S(z^*) = \sum_j x_i^j c_k^j(\tilde{z}_k^{j*}) \geq c_i(z^*)$$
and in general
$$\sum_j x_i^j c_k^j(\tilde{z}_k^{j*}) \geq c_i(z^*).$$
We can compute that: 
\begin{align*}
    \sum_i \sum_k \sum_j x_i^j c_k^j(\tilde{z}_k^{j*})(z_{k,i} - z_{k,i}^*)
    &= \sum_i \sum_k \sum_j x_i^j c_k^j(\tilde{z}_k^{j*}) z_{k,i} - \sum_i \sum_k \sum_j x_i^j c_k^j(\tilde{z}_k^{j*}) z_{k,i}^* \\
    &= \sum_i \sum_k z_{k,i} \sum_j x_i^j c_k^j(\tilde{z}_k^{j*}) - \sum_k \sum_j c_k^j(\tilde{z}_k^{j*}) \sum_i x_i^j  z_{k,i}^* \\
    &= \sum_i \sum_k z_{k,i} \sum_j x_i^j c_k^j(\tilde{z}_k^{j*}) - \sum_k \sum_j c_k^j(\tilde{z}_k^{j*}) \tilde{z}_k^{j*} \\
    &= \sum_i \sum_k z_{k,i} \sum_j x_i^j c_k^j(\tilde{z}_k^{j*}) - \textnormal{SC}(z^*) \\
    &= \sum_i \sum_k z_{k,i} \sum_j x_i^j c_k^j(\tilde{z}_k^{j*})  - \sum_i c_i(z^*) n_i \\
    & \geq \sum_i \sum_k z_{k,i} c_i(z^*) - \sum_i c_i(z^*) n_i \\
    & \geq \sum_i c_i(z^*) \sum_k z_{k,i} - \sum_i c_i(z^*) n_i \\
    &= \sum_i c_i(z^*) n_i - \sum_i c_i(z^*) n_i \\
    &= 0.
\end{align*}
($\longleftarrow$) Let $z^*$ be a distribution that verifies
$$  \sum_i \sum_k \sum_j x_i^j c_k^j(\tilde{z}_k^{j*})(z_{k,i} - z_{k,i}^*) \geq 0 \;\;, \forall z \in Z.$$ 
By contradiction, let us suppose that $z^*$ is not a Wardrop equilibrium. Then by definition, there exists a player type $i$ and two strategies $S_1 \in S_i$, $S_2 \in S_i$ such that $z^*_{S_1} >0$ and $c_{S_1}(z^*) > c_{S_2}(z^*)$. We construct the distribution $z$ such that $z_S = z_S^*$ for all $S$ such that $S \neq S_1$ and $S \neq S_2$. We impose $z_{S_1} = 0$ and  $z_{S_2} = z^*_{S_1} + z^*_{S_2}$. One can verify that $z$ is an acceptable distribution. Let us suppose that $S_1 \in p_k$ and $S_2 \in p_{k'}$. Then, we have $k \neq k'$ because $c_{S_1}(z^*) > c_{S_2}(z^*)$ (otherwise we would have $c_{S_1}(z^*) = c_{S_2}(z^*)$). We also have that $z_{k',i} = z^*_{k',i} + z^*_{S_1}$ while $z_{k,i} = z^*_{k,i} - z^*_{S_1}$. Furthermore, it holds that $c_{S_1}(z^*) = \sum_j x_i^j c_k^j(\tilde{z}_k^{j*}) > c_{S_2}(z^*) = \sum_j x_i^j c_{k'}^j(\tilde{z}_{k'}^{j*})$. As $i''\neq i$ or $k'' \neq k \neq k'$, we have $z_{k'',i''} - z_{k'',i''}^* = 0$. We can compute that: 
\begin{align*}
    \sum_i \sum_k \sum_j x_i^j c_k^j(\tilde{z}_k^{j*})(z_{k,i} - z_{k,i}^*) 
    &= \sum_j x_i^j c_{k'}^j(\tilde{z}_{k'}^{j*})(z_{k',i} - z_{k',i}^*) + \sum_j x_i^j c_k^j(\tilde{z}_k^{j*})(z_{k,i} - z_{k,i}^*) \\
    &= \sum_j x_i^j c_{k'}^j(\tilde{z}_{k'}^{j*})z^*_{S_1} + \sum_j x_i^j c_k^j(\tilde{z}_k^{j*})(-z^*_{S_1}) \\
    &= c_{S_2}(z^*) z^*_{S_1}  - c_{S_1}(z^*) z^*_{S_1} \\
    &= z^*_{S_1} (c_{S_2}(z^*) - c_{S_1}(z^*))\\
    &< 0
\end{align*}
that leads to a contradiction. 
\end{proof}
Let us remind that we consider \textnormal{loss} functions of the form $\text{\textnormal{loss}}(W) = \sum_j \sum_k \ell(\xi,y_k^j)$ where $\ell(\xi,y_k^j) = A_k^j \xi^{\beta}$ with $A_k^j > 0$, ${\beta \geq 2}$.
For such \textnormal{loss} functions, we can establish the following result.
\begin{lem}
Let $\text{\textnormal{loss}}(W) = \sum_j \sum_k \ell(\xi,y_k^j)$ where $\ell(\xi,y_k^j) = A_k^j \xi^{\beta}$ with $A_k^j > 0$, ${\beta \geq 2}$ and $b^* \in B$. Then,
$$\forall b \in B,\  \langle \nabla \text{\textnormal{loss}}(b^*), b-b^* \rangle \geq 0 \text{ if and only if }
    \forall z \in Z,\ \sum_i \sum_k \sum_j x_i^j c_k^j(\tilde{b}_k^{j*})(z_{k,i} - b_{k,i}^*) \geq 0.$$
\end{lem}
\begin{proof}
We start by showing that the considered \textnormal{loss} functions ensure that the cost functions of the associated congestion game are solutions of the differential equation: 
$$(\beta -1) c_k^j(t) = c_k^{j\prime} (t) t.$$
In fact, the solutions of the equation $(\beta -1) U(\xi) =U'(\xi) \xi$ which verify $U(0)=0$ are functions $U$ such that $U(\xi) = A_U \xi^{(\beta -1)}$ with ${\beta \geq 2}$. We restrict the set of solutions to functions $U$ such that $A_U>0$ so that $U$ is increasing. Then we have, for all $k$ and $j$, $c_k^j(\xi) = A_k^j \xi^{\beta -1} $ with $A_k^j> 0$, ${\beta \geq 2}$. The fact that $\ell(\xi, y_k^j) = c_k^j(\xi) \xi$ imposes that $\ell$ has the form: $\ell(\xi,y_k^j) = A_k^j \xi^{\beta}$ with $A_k^j> 0$, ${\beta \geq 2}$. 
The \textnormal{loss} function $\textnormal{loss}(W) = \sum_j \sum_k \ell(\xi,y_k^j)$ where $\ell(\xi,y_k^j) = A_k^j \xi^{\beta}$ with $A_k^j > 0$ respects all the conditions, \ie, $\frac{\ell}{x}$ is increasing, positive and continuous. It also allows the cost functions of the associated congestion game to verify the differential equation written above. 
For this type of \textnormal{loss}, we can rewrite the condition: 
\begin{align*}
    \langle \nabla \textnormal{loss}(b^*) , b-b^* \rangle &= \sum_i \sum_k \sum_j x_i^j c_k^{j\prime} (\tilde{b}_k^{j*}) \tilde{b}_k^{j*} (b_{k,i} - b_{k,i}^*) + \sum_i \sum_k \sum_j x_i^j c_k^j(\tilde{b}_k^{j*})(b_{k,i} - b_{k,i}^*) \\
    &= \sum_i \sum_k \sum_j x_i^j (\beta -1) c_k^j(\tilde{b}_k^{j*}) (b_{k,i} - b_{k,i}^*) + \sum_i \sum_k \sum_j x_i^j c_k^j(\tilde{b}_k^{j*})(b_{k,i} - b_{k,i}^*) \\
    &= \beta \sum_i \sum_k \sum_j x_i^j  c_k^j(\tilde{b}_k^{j*})(b_{k,i} - b_{k,i}^*).
\end{align*}
Then, using the fact that $B=Z$, we have:
$$\forall b \in B,\  \langle \nabla \textnormal{loss}(b^*), b-b^* \rangle \geq 0 \iff  \forall z \in Z,\ \sum_i \sum_k \sum_j x_i^j c_k^j(\tilde{b}_k^{j*})(z_{k,i} - b_{k,i}^*) \geq 0.$$ 
\end{proof}
The proof of the theorem is now straightforward. If $W$ is such that the family $b^*$ induced by $W$ is a local minimum of the \textnormal{loss} function $\text{\textnormal{loss}}(W)$ of $N$, then 
$$ \langle \nabla \text{\textnormal{loss}}(b^*) , b-b^* \rangle \geq 0, \quad \forall b \in B $$ 
and  
$$\sum_i \sum_k \sum_j x_i^j c_k^j(\tilde{b}_k^{j*})(z_{k,i} - b_{k,i}^*) \geq 0, \quad \forall z \in Z$$ which implies that
$z_{W}$, the flow associated to $W$ such that $z^*_{k,i} = b^*_{k,i}$, is a Wardrop equilibrium of the associated congestion game. This concludes the proof of the main theorem.
\end{proof}

\section*{Proof of Corollary \ref{corollary:binary}}
\begin{proof}
For a set of weights $W$, we can express the squared \textnormal{loss} denoted by $S\textnormal{loss}$ as follows:
$$S\text{\textnormal{loss}}(W) = \sum_j \left(\sum_{k \ne \mathsf{e}_j}  \left(\tilde{b}_k^j\right)^2    +   \left(1-\tilde{b}_{\mathsf{e}_j}^j\right)^2\right) $$
while for $\textnormal{loss}(\cdot)$ we have:
$$\text{\textnormal{loss}}(W) = \sum_j \sum_{k \ne \mathsf{e}_j}  (\tilde{b}_k^j)^2.$$ 
One can see that:
$$S\text{\textnormal{loss}}(W) = \text{\textnormal{loss}}(W) +  const,$$
where $const = \sum_j  (1-\tilde{b}_{\mathsf{e}_j}^j)^2$. 
The squared \textnormal{loss} penalizes the outputs equal to 1. Let us rewrite $const$ recalling that $\sum_i  x_i^j  = 1$ and $\sum_k z_{k,i} = 1$.
\begin{align*}
\tilde{b}_{\mathsf{e}_j}^j &= \sum_i x_i^j b_{\mathsf{e}_j,i} = \sum_i x_i^j \left(1 -\sum_{k \ne \mathsf{e}_j} b_{k,i} \right) \\
			&	= \sum_i  x_i^j -  \sum_i  x_i^j \sum_{k \ne \mathsf{e}_j} b_{k,i} = 1 - \sum_{k \ne \mathsf{e}_j} \sum_i b_{k,i} x_i^j \\
			&	= 1 - \sum_{k \ne \mathsf{e}_j} \tilde{b}_k^j.
\end{align*}
Then,
\begin{align*}
 const &= \sum_j  (1-\tilde{b}_{\mathsf{e}_j}^j)^2 \\
 &= \sum_j  \left(1-\left(1 - \sum_{k \ne \mathsf{e}_j} \tilde{b}_k^j \right)\right)^2 \\
 &= \sum_j \left(\sum_{k \ne \mathsf{e}_j} \tilde{b}_k^j \right)^2.
\end{align*} 
Finally, we have that 
$$S\text{\textnormal{loss}}(W) = \text{\textnormal{loss}}(W) +  \sum_j \left(\sum_{k \ne \mathsf{e}_j} \tilde{b}_k^j \right)^2.$$ 
In a general case of classification with a number of classes $C\geq3$, nothing guarantees that $S\text{\textnormal{loss}}(\cdot)$ and $\text{\textnormal{loss}}(\cdot)$ have the same local minima. However, in the case of binary classification with $C=2$, we have that:
\begin{align*}
    const &= \sum_j \left(\sum_{k \ne \mathsf{e}_j} \tilde{b}_k^j\right)^2 = \sum_j \sum_{k \ne \mathsf{e}_j}  (\tilde{b}_k^j)^2 =\text{\textnormal{loss}}(W).
\end{align*}
This result is obtained due to the fact that the set $\{k \ne \mathsf{e}_j\}$ has an only element ($1 \leq k \leq C =2 $). This implies that $S\text{\textnormal{loss}}(W) = 2\text{\textnormal{loss}}(W)$ and we have that $S\textnormal{loss}$ and $\textnormal{loss}$ have the same local minima. 
\end{proof}

\section*{Proof of Lemma \ref{lemma_2_nonlinear}}
\begin{proof}
Under the assumptions of Lemma 3, the \textnormal{loss} of the non-linear DNN can be written as: 
$$\textnormal{loss}(W) = \sum_j \sum_k \ell(\mathbb{E}(o_k^j),y_k^j).$$
We further have that
$$\mathbb{E}(o_k^j) = \rho \sum_{p \in p_k} x^j_p  w_p = \rho \tilde{b}_k^{j}.$$
Then,
\begin{align*}
\textnormal{loss}(W) &= \sum_j \sum_k \ell(\rho \,  \tilde{b}_k^{j},y_k^j) \\
&= \sum_j \sum_k A_k^j (\rho \,  \tilde{b}_k^{j})^{\beta} \\
&= \rho^{\beta} \sum_j \sum_k A_k^j (\tilde{b}_k^{j})^{\beta} \\
&= \rho^{\beta} \textnormal{loss}'(W),
\end{align*}
where $\textnormal{loss}'(W)$ is the \textnormal{loss} of the linear network associated to $N$. It follows that $\textnormal{loss}$ and $\textnormal{loss}'$ have the same local minimums.
\end{proof}

\newpage
\bibliography{_jair}

\end{document}